%% file: main.tex
\documentclass{article}

\usepackage{microtype}
\usepackage{graphicx}
\usepackage{subfigure}
\usepackage{booktabs} %

\usepackage{hyperref}

\usepackage{amsmath,amssymb,amsfonts, amsthm}
\usepackage{cleveref}
\usepackage{algorithm}
\usepackage[noend]{algpseudocode}
\usepackage{xcolor}
\usepackage{multirow}
\usepackage{tikz}
\let\emptyset\varnothing        %
\newcommand{\CO}{\mathcal{O}}   %
\theoremstyle{definition}       
\newtheorem*{theorem}{Theorem}
\newtheorem{definition}{Definition}
\newtheorem{corollary}{Corollary}
\newtheorem{lemma}{Lemma}
\newcommand*\circled[1]{\tikz[baseline=(char.base)]{\node[shape=circle,fill,inner sep=1pt] (char) {\textcolor{white}{#1}};}}

\definecolor{mitred}{RGB}{150, 44, 54}

\usepackage[accepted]{mlsys2021}

\mlsystitlerunning{IOS: Inter-Operator Scheduler for CNN Acceleration}

\begin{document}

\twocolumn[
\mlsystitle{IOS: Inter-Operator Scheduler for CNN Acceleration}

\mlsyssetsymbol{equal}{*}

\begin{mlsysauthorlist}
\mlsysauthor{Yaoyao Ding\textsuperscript{*}}{toronto,vi}
\mlsysauthor{Ligeng Zhu}{mit}
\mlsysauthor{Zhihao Jia}{cmu}
\mlsysauthor{Gennady Pekhimenko}{toronto,vi}
\mlsysauthor{Song Han}{mit}
\end{mlsysauthorlist}

\mlsysaffiliation{toronto}{University of Toronto}
\mlsysaffiliation{mit}{Massachusetts Institute of Technology}
\mlsysaffiliation{cmu}{Carnegie Mellon University}
\mlsysaffiliation{vi}{Vector Institute}

\mlsyscorrespondingauthor{Song Han}{songhan@mit.edu}

\mlsyskeywords{Deep Neural Network, Scheduling, Parallelization, Inference, Acceleration}

\vskip 0.3in

\input{sections/0_abstract.tex}] %

\printAffiliationsAndNotice{\textsuperscript{*}Work done while interning at MIT HAN Lab.}  %

\input{sections/1_introduction.tex}

\input{sections/2_background.tex}

\input{sections/3_problem_definition.tex}

\input{sections/4_methods.tex}

\input{sections/5_experiments.tex}
\input{sections/6_conclusions.tex}

\input{sections/7_acknowledgements.tex}

\bibliography{references/references,references/distributed,references/framework,references/cnnarch,references/nas,references/translation}
\bibliographystyle{mlsys2021}

\input{sections/8_appendix.tex}

\end{document}

%% file: sections/0_abstract.tex
\begin{abstract}

To accelerate CNN inference, existing deep learning frameworks focus on optimizing {\em intra-operator} parallelization. However, a single operator can no longer fully utilize the available parallelism given the rapid advances in high-performance hardware, resulting in a large gap between the peak performance and the real performance. 
This performance gap is more severe under smaller batch sizes. 
In this work, we extensively study the parallelism \textit{between} operators and propose Inter-Operator Scheduler (IOS) to automatically schedule multiple operators' parallel execution through a novel dynamic programming algorithm. 
IOS consistently outperforms state-of-the-art libraries (e.g., TensorRT) by $1.1$ to $1.5\times$ on modern CNN benchmarks.
The code to reproduce each experiment is available at: {\footnotesize\url{https://github.com/mit-han-lab/inter-operator-scheduler}}.

\end{abstract}

%% file: sections/1_introduction.tex
\section{Introduction}

Convolutional neural networks (CNNs) have achieved state-of-the-art performance across many tasks, including computer vision~\cite{krizhevsky2012alexnet, he2016deep}, machine translation~\cite{sutskever2014seq2seq, devlin2018bert}, and game playing~\cite{mnih2013playing,silver2016mastering}.
The success comes at the cost of growing computational requirements. The high demand for computation makes efficient inference more critical in real deployment~\cite{han2015deep,chen2018tvm,taso}.

A common practice to improve inference efficiency is parallelization.
Deep learning frameworks such as Tensorflow~\cite{abadi2016tensorflow} and Pytorch \cite{paszke2017pytorch} exploit {\em intra-operator parallelism}, which parallelizes arithmetic operations within a {\em single} CNN operator (e.g., convolution).
\begin{figure}[!ht]
    \hspace{15pt}
    \includegraphics[width=0.8\linewidth]{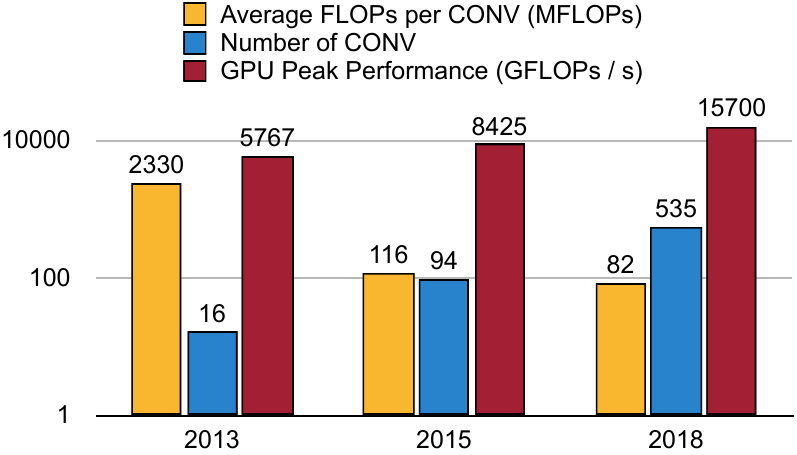}\vspace{-2pt}
    \caption{The trends of average computation per convolution, number of convolutions in a CNN and hardware peak performance. Device peek performance increases while average computation per convolution decreases, leading to a larger utilization gap. VGGNet and GTX 980Ti, Inception V3, and GTX 1080, NASNet and Tesla V100 are chosen as representatives for 2013, 2015, and 2018 respectively. All FLOPs are measured for single precision.
    }
    \label{fig:comp_change}
    \vspace{-15pt}
\end{figure}
However, due to the rapid advances in high-performance hardware, intra-operator parallelism is no longer sufficient to obtain efficient resource utilization. 
As shown in Figure~\ref{fig:comp_change}, the peak FP32 performance of a GPU has increased from 5.8 TFLOPs/s in 2013 to 15.7 TFLOPs/s in 2018 (shown in red). 
NVIDIA Tesla A100 even reaches a peak FP32 performance of 19.5 TFLOPs/s.

Meanwhile, there is a recent trend in CNN design to replace a single branch of convolutions with multiple branches of convolutions, which is advantageous due to increased model capacity under a fixed computation budget \cite{szegedy2016rethinking,zoph2018learning,xie2019exploring}. As a result, the number of convolutions grows while the computation FLOPs in each convolution becomes smaller. For example, the average floating-point operations (FLOPs) per convolution has decreased from 2330 MFLOPs/kernel in VGG to 82 MFLOPs/kernel in NASNet. This exacerbates the device's under-utilization problem.

\begin{figure*}[ht]
    \centering
    \includegraphics[width=1.0\textwidth]{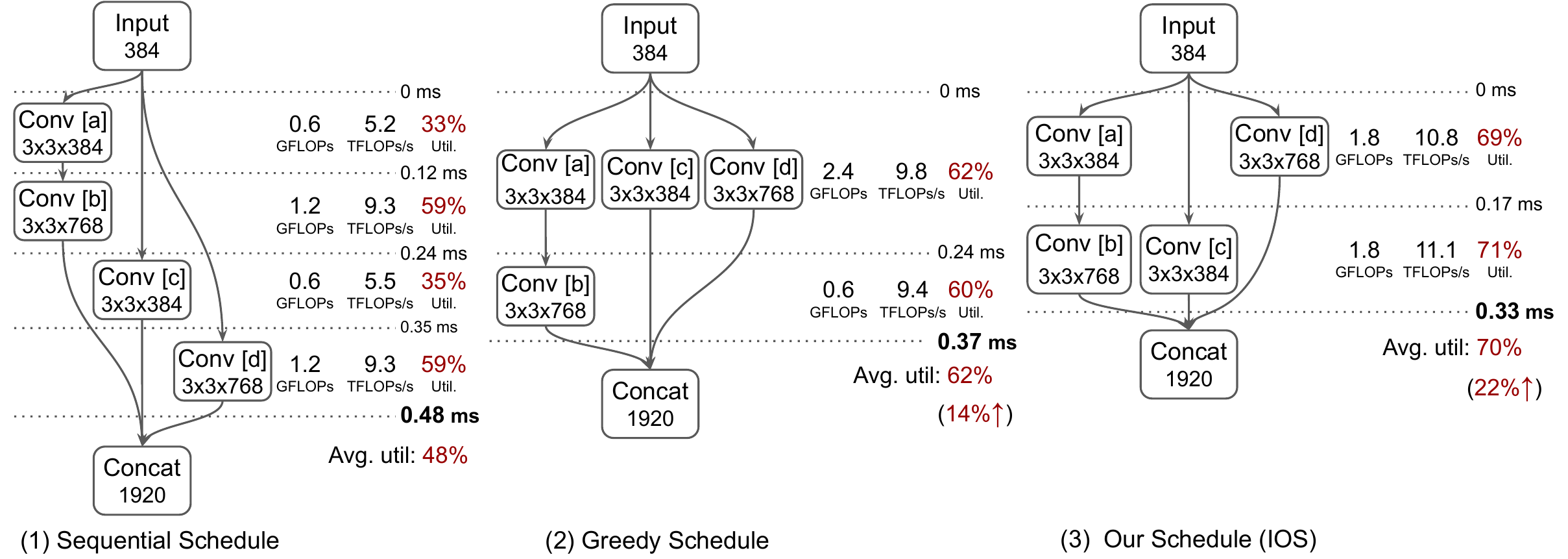}
    \vspace{-18pt}
    \caption{Different execution schedules for a computation graph on NVIDIA Tesla V100 GPU.  Operators scheduled to run in parallel are placed at the same level between two dotted lines called a \textit{stage}. Computation (GFLOPs), performance (TFLOPs/s), and hardware utilization (\%) for each stage are profiled on the right. Both sequential and greedy schedules result in low resource utilization (48\%-62\%) and high latency (0.37-0.48ms).
    Our schedule yields higher utilization (70\%) and lower latency (0.33ms).} 
    \label{fig:intro_overview}
    \vspace{-5pt}
\end{figure*}
To address this problem,
recent work explores {\em inter-operator parallelism} by executing multiple CNN operators in parallel guided by different heuristics~\cite{tang2018scheduling,jia2019optimizing, rammer}. 
For example, 
MetaFlow~\cite{jia2019optimizing} fuses multiple operators matching a specific pattern into a larger operator to increase operator granularity. 
Tang et al.~\cite{tang2018scheduling} proposes a {\em greedy} strategy that directly executes all available CNN operators on CPU to maximize resource utilization.
These approaches apply different heuristics to optimize local parallelization across a few CNN operators; however, such techniques do not lead to a \emph{globally optimal} schedule for the entire CNN architecture.
For example, given an input CNN (Figure~\ref{fig:intro_overview} (1)), a greedy schedule (Figure~\ref{fig:intro_overview} (2))
would perform convolutions [a], [c], and [d] in parallel, and run convolution [b] in a subsequent stage upon the completion of the previous stage. %

This greedy schedule is sub-optimal for two reasons. First, a greedy schedule eagerly puts more operators in the early stages (as soon as they are available for execution) and fewer operators in subsequent stages, resulting in low utilization in later stages. 
Second, executing too many operators on the device concurrently may lead to resource contention problem that hurts the performance. 
For example, as shown in Figure~\ref{fig:intro_overview}, the greedy schedule (2) suffers from resource contention problem in the first stage and low-utilization problem in the second stage, comparing to our proposed schedule (3).

Obtaining an optimized schedule to parallelize a CNN model is a challenging task. 
On the one hand, the number of schedules grows exponentially with the number of operators, making it infeasible to evaluate all possible schedules exhaustively. For example, a network with $33$ operators can have $9.2\times 10^{22}$ number of feasible schedules. 
On the other hand,
an optimal schedule also depends on hardware specifications and inference settings (e.g., batch size). 
A high-end GPU (e.g., Tesla V100) can efficiently execute a schedule with many operators in parallel, while a low-end GPU (e.g., Tesla K80) might suffer from resource contention using the same schedule. A large batch size naturally offers more intra-operator parallelism, while a small batch size has a stronger need for inter-operator parallelization. 
Therefore, given a diverse set of CNN architectures, hardware, and inference settings, it is hard to devise an efficient schedule manually for all scenarios.

To address this challenge, we propose 
IOS, an inter-operator scheduler
that accelerates CNN inference by combining intra- and inter-operator parallelism. 
We observe that different schedules share common sub-schedules; thus, IOS adopts a dynamic programming technique to explore the schedule space and finds a highly optimized schedule under low search cost. 
We evaluate IOS on modern CNN models, including Inception-V3~\cite{szegedy2016rethinking}, RandWire~\cite{xie2019exploring}, NasNet-A~\cite{zoph2018learning}, and SqueezeNet~\cite{iandola2016squeezenet}. IOS consistently outperforms the sequential schedule and greedy schedule. IOS achieves $1.1$ to $1.5\times$ inference speedup compared to existing deep learning libraries (e.g., TensorRT).
Furthermore, IOS demonstrates the necessity of \textit{customizing the scheduling policy} for different hardware and inference configurations. 
IOS can achieve up to $1.15\times$ inference speedup by customizing the scheduling recipe compared to itself with no customization.

Our contributions are summarized as follows:

\begin{itemize}
    \item We point out a major bottleneck for efficient CNN inference: existing intra-operator parallelism cannot saturate modern hardware's high parallelism, especially for recent multi-branch CNN models. Inter-operator parallelism is crucial.
    \item We propose a novel dynamic programming algorithm to find a highly optimized schedule for inter-operator parallelization. This technique is platform-agnostic and can serve as a general technique 
    for popular frameworks such as TensorFlow~\cite{tensorflow2015-whitepaper} and TVM~\cite{chen2018tvm}.
    \item We apply IOS to various hardware and inference settings and show that the different configurations require different schedules. We can automatically customize the scheduling policy for different hardware and inference configurations.
    The specialized schedules 
    consistently outperform existing deep learning libraries with $1.1$ to $1.5\times$ measured speedup in inference.
\end{itemize}

%% file: sections/2_background.tex
\section{Background and Related Work}

\textbf{CNN Design.} 
Several lightweight design primitives have been recently introduced to improve the efficiency of CNNs.
Examples include SequeezeNet~\cite{iandola2016squeezenet}, MobileNet~\cite{sandler2018mobilenetv2} and ShuffletNet~\cite{zhang2018shufflenet}.
However, such design patterns cannot fully utilize the hardware. Hardware under-utilization becomes more severe as accelerators are getting more powerful (shown in Figure~\ref{fig:comp_change}). 
On the other hand, multi-branch CNNs become a trend in model architecture design, including both manually designed networks \cite{szegedy2015going, iandola2016squeezenet, szegedy2016rethinking} and the networks discovered by neural architecture search~\cite{Cai2018PathLevelNT,zoph2018learning}. With a fixed computation budget, multi-branch CNNs use more small convolution primitives, which further amplifies the resource under-utilization problem on modern hardware.

\textbf{Intra-operator Parallelism.}  Current deep learning frameworks (e.g., TensorFlow and PyTorch) generally focus on intra-operator parallelism, which executes arithmetic operations within a {\em single} operator in parallel (e.g., tiled matrix multiplication). 
Tensorflow and PyTorch are built upon vendor-provided libraries (e.g., cuDNN), a set of DNN compute primitives heavily optimized by vendor engineers to achieve near-peak machine performance. 
However, these DNN operators are executed sequentially on a hardware device.
The degree of parallelism within an operator is limited; thus, intra-operator parallelism cannot provide sufficient parallelizable computation to feed powerful hardware devices. As a result, the hardware is often under-utilized using these frameworks.

Different from manual performance tuning, Auto-Halide~\cite{mullapudi2016auto_halide}, TVM~\cite{chen2018tvm} and Ansor~\cite{ansor} exploit intra-parallelism through automatically \textit{learning} efficient schedule for individual DNN kernels. This automation saves a large amount of engineering effort and can generate more efficient DNN kernels than the manually designed counterparts. However, still, all these libraries only focus on intra-operator parallelism but do not exploit inter-operator parallelism. 

\textbf{Inter-Operator Scheduling.}
Recent work has explored inter-operator scheduling. 
Tang et al.~\cite{tang2018scheduling} proposes a greedy heuristic approach, Graphi, that executes all available CNN operators whenever possible to saturate CPU's computation capability. 
The greedy strategy does not 
\emph{holistically} optimize the computation graph's performance, hence yields unbalanced and sub-optimal schedules. 
Rammer\cite{rammer} optimizes the execution of DNN workloads by holistically exploiting parallelism through inter- and intra- operator co-scheduling, enabling a richer scheduling space for executing a DNN model. IOS focuses on the inter-operator scheduling and leaves the intra-operator scheduling to the hardware.
Nimble\cite{nimble} is a DNN engine that supports parallel execution of DNN operators on GPU and minimizes the scheduling overhead using ahead-of-time (AOT) scheduling. The scheduling algorithm used in Nimble does not consider the latency of each operator, while IOS is a profile-based scheduler.

\textbf{Graph transformation.}
MetaFlow~\cite{jia2019optimizing} performs functional-preserving graph transformations to optimize DNN architectures. 
Merging operators with the same input enables more parallelism (a larger operator compared to two small sequential operators) and reduces accesses to GPU memories.
TASO~\cite{taso} further introduces an automated generation of substitution rules  and it explores more mathematically equivalent DNN architectures of the input one comparing to MetaFlow. 
MetaFlow and TASO consider the whole computation graph and search for highly optimized substitution strategies.
However, the inter-oprator parallelism utilized by MetaFlow and TASO is still limited as 
only the same type of operators can be merged. 

To address the large schedule space problem, IOS utilizes dynamic programming to take advantage of the common sub-schedules among different schedules. Also, IOS supports concurrent execution of different types of operators, addressing the limitation of MetaFlow and TASO.

%% file: sections/3_problem_definition.tex
\section{Problem Definition}
\label{sec:preliminary_definition}

This section defines the {\em schedule} in IOS and formulates the problem.

\begin{figure}[ht]
    \centering
    \includegraphics[width=1.0\linewidth]{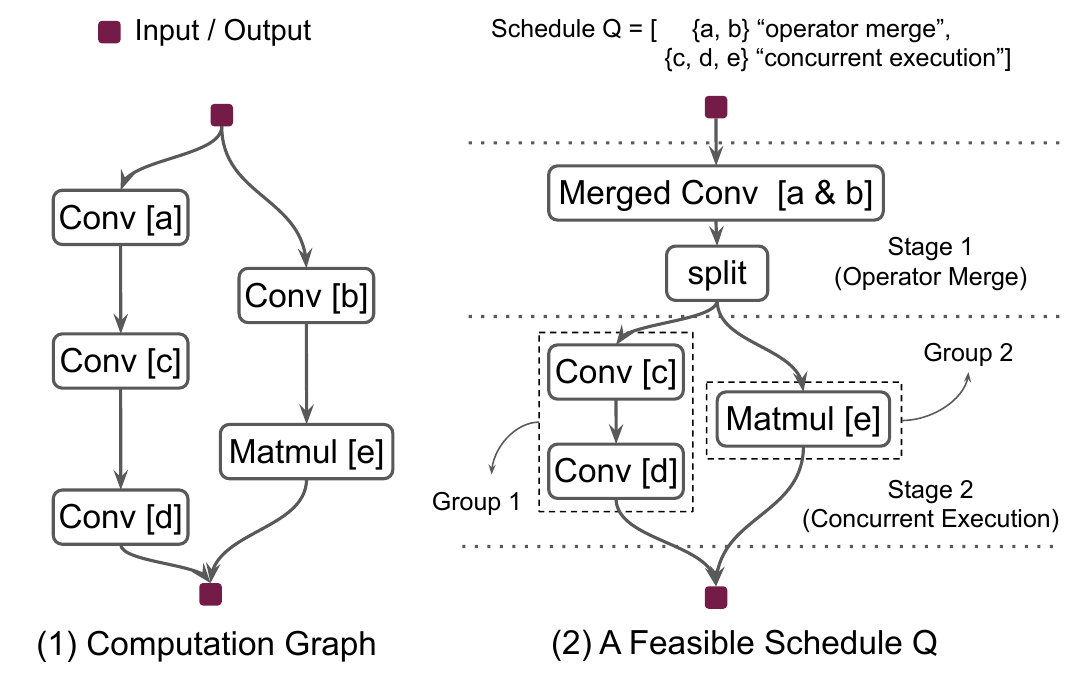}\vspace{-15pt}
    \caption{For a given \textit{computation graph} (left), a possible \textit{schedule} is shown to the right. There are five operators in the graph: convolutions a-d and matrix multiplication e. The schedule partitions operators into 2 \textit{stages}. The first stage merges convolution a and b into a larger convolution; this parallelization strategy is named \textit{operator merge}.  The second stage partitions operator c, d and e into two \textit{groups}, \{c, d\} and \{e\}. The operators in the same group are executed sequentially while different groups in the same stage are executed concurrently. This parallelization strategy is named \textit{concurrent execution}. Stages are executed one-by-one.}
    \vspace{-10pt}
    \label{fig:methods_definition}
\end{figure}

\textbf{Computation Graph.} 
A CNN is defined by a computation graph $G = (V, E)$, where $V$ is the set of operators, and $E$ is the edge set representing dependencies. A computation graph is a directed acyclic graph (DAG). Each operator in the graph represents an operator such as convolution and matrix multiplication. Each edge $(u,v)$ is a tensor that is an output of operator $u$, and an input of operator $v$. Figure~\ref{fig:methods_definition} (1) shows the computation graph of a simple CNN. 

\textbf{Stage.} To take advantage of inter-operator parallelism in a CNN architecture, its computation graph is partitioned into multiple stages. Stages are executed sequentially and the operators in the same stage are executed according to a certain parallelization strategy (see below).
Figure~\ref{fig:methods_definition} (2) shows a possible schedule that partitions the input graph into two stages, where the first stage contains operator a and b, and the second stage contains operator c, d, and e.
The parallelization strategy is discussed below.

\textbf{Parallelization Strategy.} 
Each stage adopts one of the following two parallelization strategies: \textit{operator merge} and \textit{concurrent execution}. 
ISO considers both of them and automatically picks the more efficient one for each stage.
The choice depends on operator types, input tensor shapes, and the hardware device to perform CNN computations. 

To be eligible for \textit{operator merge}, 
the operators' type must be the same while the hyperparameters can be different.
For example, two convolutions with the same stride but different kernel sizes can be merged. The smaller kernel will be padded with zeros to fit the large kernel, so we can stack their kernels together. 
In Figure \ref{fig:methods_definition} (1), if \texttt{Conv[a]} has 128 3x3 kernels while \texttt{Conv[b]} has 256 3x3 kernels, we can stack their kernels together and replace \texttt{Conv[a]} and \texttt{[b]} by a \texttt{Merged Conv[a\&b]} with 384 3x3 kernels. 
Besides increasing parallelism, it also reduces the memory accesses to the input tensor from twice to only once. A split operator is required to partition the merged convolution's output to recover the original outputs of \texttt{Conv[a]} and \texttt{Conv[b]}.

Under \textit{concurrent execution}, the operators in the stage are partitioned into disjoint {\em groups}. More specifically, 
if two operators are connected by an edge, they are partitioned into the same group.
Different groups within the same stage are executed concurrently, while the operators within the same group are executed sequentially.
IOS considers simultaneous executions of operators with {\em different} types.
In the second stage of Figure~\ref{fig:methods_definition} (2), the three operators are partitioned into two groups. The first group contains operators \texttt{Conv[c]} and \texttt{Conv[d]} while the second group contains operator \texttt{Matmul[e]}. The two groups are executed concurrently while \texttt{Conv[c]} and \texttt{Conv[d]} are executed sequentially in their group.

\textbf{Schedule.} We define a \textit{schedule} $Q$ of a computation graph $G$ as 
$
    Q = \{(S_1, T_1), (S_2, T_2), \dots, (S_k, T_k)\},
$
where $S_i$ is the set of operators in the $i$th stage and $T_i$ is the corresponding parallelization strategy, either ``concurrent execution'' or ``operator merge''. For example, the schedule for Figure~\ref{fig:methods_definition} (2) is: {$Q =$ $\{(\{a, b\}, \text{operator merge})$, $(\{c, d, e\}, \text{concurrent execution})\}$}. The schedule $Q$ executes the network from the first stage $(S_1, T_1)$ to the last stage $(S_k, T_k)$ sequentially.
$S_i$ may contain only one operator if it is the best choice (e.g., a very large operator that saturates the entire GPU). %

\textbf{Problem Formulation.} Let $c$ be a cost function defined on a computation graph $G$ and schedule $Q$. We aim to find a schedule $Q^*$ to minimize the cost function for a given computation graph $G$, i.e., $Q^* = \text{argmin}_{Q}c(G,Q)$. 
In this work, the cost function $c(G, Q)$ is defined as the latency of running $G$ following schedule $Q$.

%% file: sections/4_methods.tex
\section{Methods}
This section introduces our Inter-Operator Scheduler (IOS) in three parts. 
Section~\ref{sec:ios} elaborates the IOS design in details. 
Section~\ref{sec:method_algorithm_complexity} analyzes the time complexity of IOS.
Finally, Section~\ref{sec:method_pruning_strategy} introduces the pruning optimizations to reduce the search time of IOS.

\subsection{Inter-Operator Scheduler (IOS)}
\label{sec:ios}
\label{sec:method:dynamic_programming_scheduler}

\begin{figure}[ht]
    \centering
    \includegraphics[width=1.0\linewidth]{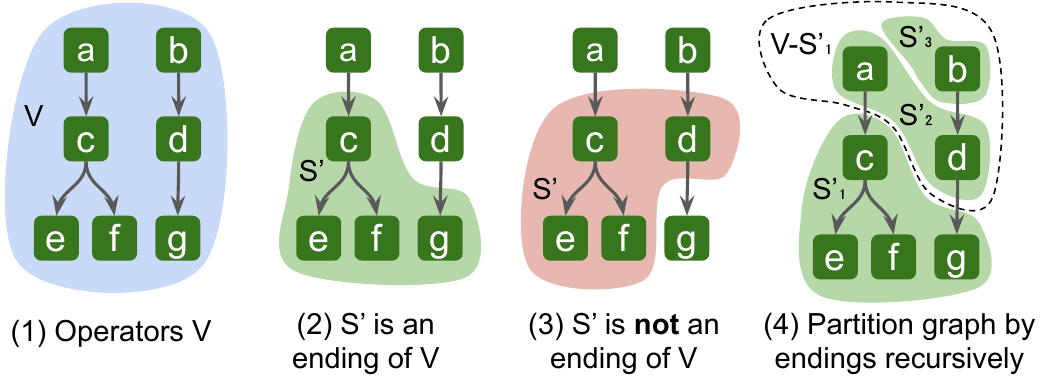}\vspace{-15pt}
    \caption{The illustration of {\em ending}. (1) shows all the operators $V$. $S'$ in (2) is an ending of $V$. However, $S'$ in (3) is not an ending of $V$ because there is an edge from d to g (from $S'$ to $V-S'$). We can partition a graph by selecting an ending for remaining operators recursively, as shown in (4), where $S'_1$ is an ending of $V$ while $S'_2$ is an ending of $V-S'_1$.}
    \label{fig:methods_ending}
\end{figure}

To find an optimized schedule for a CNN architecture, we first partition its computation graph $G = (V, E)$ into $V-S'$ and $S'$, where all edges between $V-S'$ and $S'$ start from $V-S'$ and end in $S'$.
Such $S'$ is called an \textit{ending} of $V$, as illustrated in Figure~\ref{fig:methods_ending}. 
There can be many endings of $V$. 
The last stage's operators in $V$'s optimal schedule must be an ending of $V$. We can enumerate the ending $S'$ of $V$ and convert the original problem to a sub-problem that finds the optimal schedule for $V-S'$.
The whole graph can be scheduled by applying the partition recursively.

Let $\text{cost}[S]$ be the latency of an optimal schedule for $S$. Let $\text{stage\_latency}[S']$ be the latency of stage $(S', T)$ where $T$ is the better parallelization strategy for $S'$ among the two possible ones. We formalize this idea as follows,
$$
    \text{cost}[S] = \min_{S'}(\text{cost}[S-S'] + \text{stage\_latency}[S']),
$$
where $S'$ is an ending of $S$, and $\text{cost}[\emptyset] = 0$. Finally, $\text{cost}[V]$ is the latency of an optimal schedule for the entire computation graph $G$. To construct the optimal schedule we found, we record the corresponding $S'$ that minimizes the latency for each $S$ (i.e., $\text{cost}[S]$) in $\text{choice}[S]$.

With this general idea, we implement IOS in three functions  \textproc{InterOperatorScheduler} (L3-12), \textproc{Scheduler} (L13-22) and \textproc{GenerateStage} (L23-33), as shown in Algorithm~\ref{alg}. \textproc{InterOperatorScheduler} takes a computation graph as an input and returns the optimal schedule found by IOS. \textproc{Scheduler} is a recursive function implementing the dynamic programming algorithm to find the optimal schedule for a subset of operators in $G$. \textproc{GenerateStage} chooses a better parallelization strategy for given operators $S'$.

\begin{algorithm}[!tb]
    \small
    \caption{Inter-Operator Scheduler (IOS)}
    \label{alg}
    \begin{algorithmic}[1]
        \Statex \textbf{Input: } a computation graph $G = (V, E)$, 
        \Statex \hspace{28pt} and a schedule pruning strategy $P$
        \Statex \textbf{Output: } a schedule found by IOS
        \State Let $\text{cost}[S] = \infty$ for all $S \subseteq V$ but $\text{cost}[\emptyset] = 0$
        \State Let $\text{choice}[S] = \emptyset$ for all $S \subseteq V$
        \Function{\textcolor{mitred}{InterOpeatorScheduler}}{$G$}
            \State $V = $ all operators in computation graph $G$
            \State \Call{Scheduler}{$V$}
            \State $Q = $ empty list
            \State $S = V$
            \While{$S \neq \emptyset$}
                \State $S', T = \text{choice}[S]$
                \State Insert stage $(S', T)$ before the head of $Q$
                \State $S = S - S'$
            \EndWhile
            \State \Return the schedule $Q$
        \EndFunction
        \Function{\textcolor{mitred}{Scheduler}}{$S$} 
            \If{$\text{cost}[S] \neq \infty$}
                \State \Return $\text{cost}[S]$ %
            \EndIf
            \ForAll{ending $S'$ of $S$ satisfying pruning strategy $P$} \label{alg:for_ending}
                \State $L_{S'}, T_{S'} = $ \Call{GenerateStage}{$S'$}
                \State $L_S = $ \Call{Scheduler}{$S-S'$} $+ L_{S'}$
                \If{$L_S < \text{cost}[S]$}
                    \State $\text{cost}[S] = L_S$
                    \State $\text{choice}[S] = (S', T_{S'})$
                \EndIf
            \EndFor
            \State \Return $\text{cost}[S]$
        \EndFunction
        \Function{\textcolor{mitred}{GenerateStage}}{$S'$}
            \State Partition $S'$ into disjoint groups: $S'_1, S'_2, \dots, S'_k$.
            \State $L_{concurrent} = $ latency of parallel execution of $\{S'_i\}$
            \If{operators in $S'$ can be merged}
                \State $L_{merge} = $ latency of merged operator
            \Else
                \State $L_{merge} = \infty$
            \EndIf
            \If{$L_{concurrent} < L_{merge}$}
                \State \Return $L_{concurrent}$, ``concurrent execution"
            \Else
                \State \Return $L_{merge}$, ``operator merge"
            \EndIf
        \EndFunction
    \end{algorithmic}
\end{algorithm}

\textbf{\textproc{InterOperatorScheduler}} (L3-12) is the entry function. It takes a computation graph $G$ as an input and returns an optimized schedule $Q$. 
This function calls \textproc{Scheduler} with operators $V$ as an argument (L5).
After calling \textproc{Scheduler}, the global variable $\text{cost}[S]$ stores the latency of an optimal schedule for $S$, while $\text{choice}[S]$ stores the last stage in the corresponding optimal schedule. 
Once $\text{choice}[\cdot]$ is obtained, we can construct the schedule found by IOS (L6-11). We start with an empty list as the initial state of our schedule (L6) and let $S$ be all the operators in $G$.  We inquire about the last stage $(S', T)$ of $S$ by $\text{choice}[S]$ and put it at the head of the current schedule $Q$. 
We repeat this process by letting $S = S-S'$ to get the remaining operators' schedule in all previous stages (L8-11). $S = \emptyset$ indicates that we have discovered an optimized schedule $Q$ for $G$.

\textbf{\textproc{Scheduler}} (L13-22) is the core part of our algorithm. It implements the dynamic programming algorithm recursivly, taking a subset of $V$ as the state. 
It takes a set of operators $S$ as an input and returns the minimal latency for $S$ among all schedules.
Because \textproc{Scheduler} may be called multiple times with the same argument $S$, for repeated calls, we cache the previous results $\text{cost}[S]$ to avoid redundant computations (L14-15). 
To find an optimal schedule for $S$, we enumerate its last stage operators $S'$ and reduce the problem into a sub-problem for $S-S'$ (L16-21). 
We use \textproc{GenerateStage} to choose a better parallelization strategy $T_{S'}$ for $S'$ and get the latency $L_{S'}$ (L17).
$L_S$ is the minimal latency for $S$ when taking $S'$ as the last stage's operators (L18).
We enumerate all possible endings of $S$ and record the minimal latency $L_S$ and the corresponding last stage $(S', T_{S'})$ in $\text{cost}[S]$ and $\text{choice}[S]$, respectively (L19-21).

\begin{figure*}[!htb]
\centering
\includegraphics[width=0.99\textwidth]{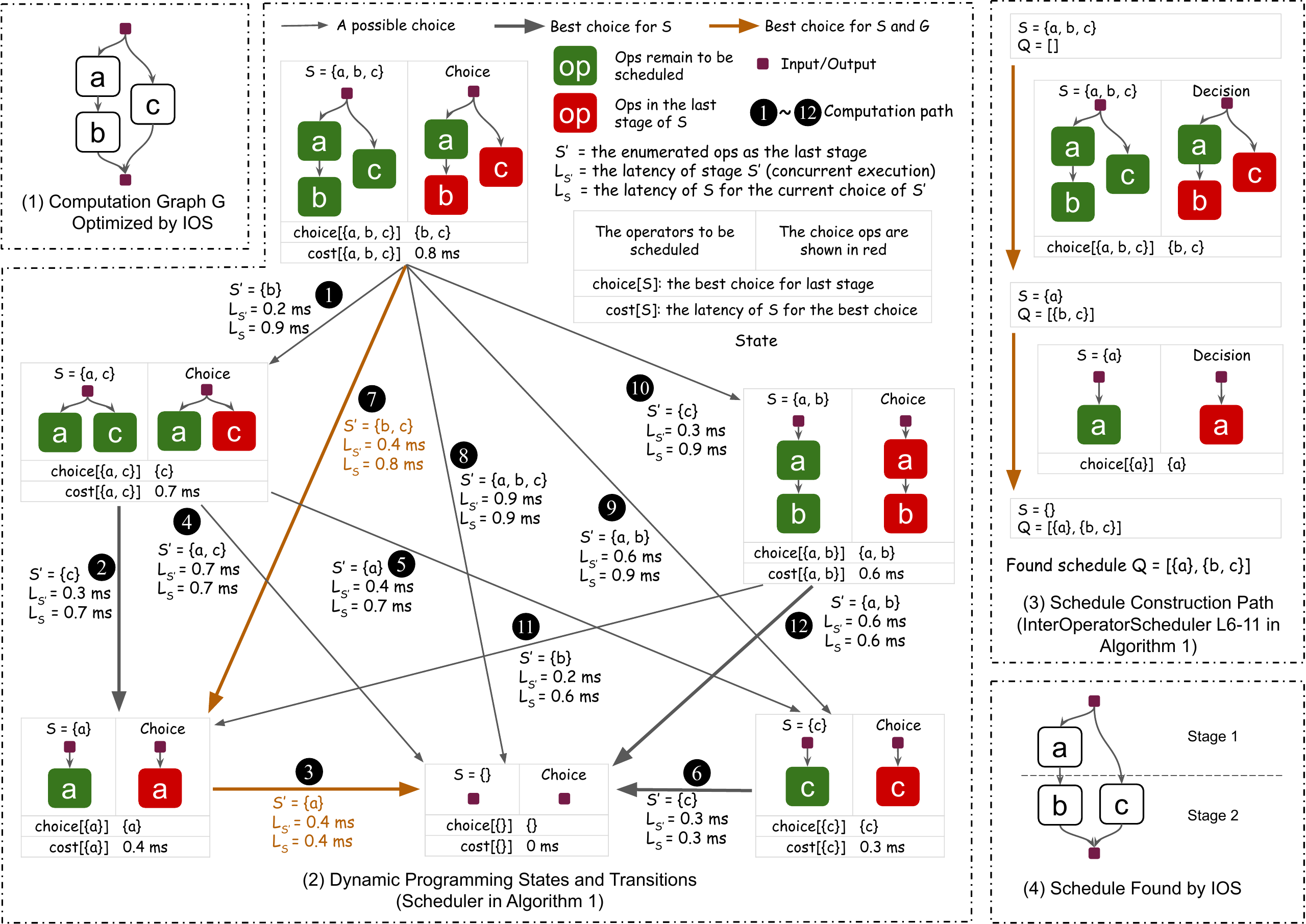}
\vspace{-5pt}
\caption{An example to illustrate how IOS finds the schedule. The computation graph to be optimized is shown in (1). It has three operators, a, b, and c, where a is followed by b, and c is independent with a and b. The states and transitions between these states are presented in (2). Here \textit{state} means the operators to be scheduled, and \textit{transition}  means the dependency between states (edges in (2)). Any path from state $S=\{a, b, c\}$ to $S = \{\}$ is corresponded with a schedule. Upon finishing the dynamic programming process (\textproc{Scheduler}), the best schedule for the computation graph can be constructed according to $\text{choice}[\cdot]$, as shown in (3). The schedule found by IOS is shown in (4). For simplicity, in this example, we only consider the concurrent execution parallelization strategy.}
\label{fig:methods_algorithm}
\vspace{-5pt}
\end{figure*}

\textbf{\textproc{GenerateStage}} (L23-33) chooses a better parallelization strategy from ``concurrent execution" and ``operator merge" for a given stage $S'$. It returns the parallelization strategy and the corresponding latency. 
It directly measures the latencies of both parallelization strategies on the hardware.
The ``concurrent execution" strategy partitions $S'$ into multiple disjoint operator groups: $S'_1, S'_2, ..., S'_k$.
Operators in different groups are executed concurrently while operators in the same group are executed sequentially. 
For the ``operator merge" strategy, if all the operators in $S'$ can be merged into a single operator (L26), we merge them 
and measure the latency of the merged operator (L27). Otherwise, we set $L_{merge}$ to infinity to force ourselves to choose the ``concurrent execution" strategy. 

Figure~\ref{fig:methods_algorithm} demonstrates how IOS discovers an optimized strategy for an input graph with three operators a, b, and c.
Figure~\ref{fig:methods_algorithm} (2) shows the dynamic programming process, the \textproc{Scheduler} in Algorithm \ref{alg}.
For simplicity, we only consider the concurrent execution parallelization strategy.
There are six \textit{states} (the operators to be scheduled, $S$) in the process. 
We start with all the operators in the computation graph as state $S = \{a, b, c\}$ (L5).
For each state $S$, \textproc{Scheduler} enumerates the ending $S'$ of $S$. 
The latency of $S$ contains two parts: latency of $S'$ as a stage and the latency of $S-S'$.
While the result of $S'$ is measured on the device directly ($L_{S'}$),
the optimal latency of $S-S'$  is obtained via solving the sub-problem recursively.
\circled{1} to \circled{\footnotesize{12}} shows the computation path. 
Note that IOS memorizes the results for each calculated state to avoid redundant computations. 
Thus, step \circled{7} visits state $S = \{a\}$, and IOS gets its latency directly (L15) 
because it has been previously visited by step \circled{2}.
\textproc{Scheduler} stores the latency ($\text{cost}[\cdot]$) and last stage ($\text{choice}[\cdot]$) in its optimal schedule.
We can construct the best schedule for the whole computation graph using $\text{choice}[\cdot]$, as shown in Figure~\ref{fig:methods_algorithm} (3). 
An optimal schedule found by IOS is shown in (4).
Both stages take ``concurrent execution" as the parallelization strategy.

\subsection{Time Complexity of IOS}
\label{sec:method_algorithm_complexity_and_pruning_strategy}
\label{sec:method_algorithm_complexity}
In this subsection, we analyze the time complexity of IOS.
We take set operations (L18, L24)
and latency measurement operations (L25, L27) as atom operations to make the analysis clear.
To analyze the time complexity of IOS, we count the number of executions of L17-21, since they dominate the whole algorithm's execution. This number equals the number of edges (i.e., transitions) in Figure~\ref{fig:methods_algorithm} (2). Furthermore, it is equivalent to count the number of pairs $(S, S')$,
where $S$ is a state and $S'$ is an ending of $S$. 
Here we define the width of a directed acyclic graph and provide the time complexity of Algorithm~\ref{alg}.

\begin{definition}[Width $d$ of a DAG]
We call $d$ the {\em width} of a directed acyclic graph $G$ if we can find at most $d$ operators in $G$ such that there is no path connecting any two of them.
\end{definition}

\begin{theorem}[Time Complexity of IOS]
The time complexity of Inter-Operator Scheduler (IOS) is $\CO(\binom{n/d+2}{2}^d)$, which can be relaxed to $\CO((\frac{n}{d}+1)^{2d})$, where $n$ is the number of operators in the computation graph and $d$ is its width.
\end{theorem}

In fact, there are computation graphs that can reach this bound, so we can not improve it without other restrictions on the schedule space. Proof can be found in Appendix~\ref{sec:appendix_time_complexity}.

\begin{table}[th]
    \centering
    \setlength{\tabcolsep}{4pt} %
    \small
    \begin{tabular}{c c c c c c c}
        \hline
        Model	& $n$ & $d$ & $\binom{n/d+2}{2}^d$ & \#$(S, S')$ & \#Schedules \\
        \\[-0.9em]
        \hline
        Inception V3& $11$    &	$6$	& $2.6\times10^4$   &   $4.9\times10^3$ &	$3.8\times10^{6}$\\
        Randwire    & $33$    &	$8$	& $3.7\times10^9$   &   $1.2\times10^6$ &	$9.2\times10^{22}$\\
        NasNet	    & $18$    &	$8$	& $5.2\times10^6$	&   $3.1\times10^5$	&   $7.2\times10^{12}$\\
        SqueezeNet	& $6$     & $3$	& $2.2\times10^2$   &	$51$	    &   $1.3\times10^2$\\
        \hline
    \end{tabular}
    \normalsize
    \caption{For the largest block of each benchmarked network, we list the number of operators $n$, the width $d$, the upper bound of transitions $\binom{n/d+2}{2}^d$, the real number of transitions \#$(S, S')$, and number of schedules.}
    \label{tab:block_info}
\end{table}
Modern convolution neural networks usually construct the network by stacking multiple blocks, making it possible to optimize each block separately. In this case,  $n$ and $d$ refers to the number of  operators within a block and the block width, rather than the full network. We list the information of the largest block for each  network benchmark in Table \ref{tab:block_info}.

The total number of feasible schedules is exponential to the number of operators (e.g., up to $9.2\times 10^{22}$ for Randwire~\cite{xie2019exploring}). Such a huge number makes it prohibitive to manually design or enumerate the schedules. 
However, by reusing the results of common sub-schedules in the schedule finding process, IOS finds the optimal schedule within $4$ hours for each network with no pruning strategy used.
The time complexity of IOS is only exponential to the width of the computation graph, which is usually very small and acceptable (e.g., $\leq 8$ in all benchmarked networks).

\subsection{Reduce the Search Time by Schedule Pruning}
\label{sec:method_pruning_strategy}

It is difficult for a dynamic programming algorithm to stop early, because it gets the best result at the very end. To reduce the search time, IOS introduces {\em schedule pruning} to reduce the exploration space by restricting the max number of groups and the max number of operators within a group.
We define the pruning strategy $P$ as a boolean function of $S$ and $S'$. We only enumerate the ending $S'$ of $S$ that satisfies the pruning strategy $P$, that is, $P(S, S') = \text{True}$ (L16 of Algorithm \ref{alg}).
The pruning strategy consists of two parameters $r$ and $s$: $P(S, S') = \text{True}$ if and only if ending $S'$ has at most $s$ groups and each group has at most $r$ operators. 
 
After applying the pruning strategy $P$, the time complexity is reduced from $\CO((\frac{n}{d}+1)^{2d})$ to $\CO((\frac{n}{d}+1)^d(r+1)^s)$. Of course, there is a trade-off between the search cost and the quality of the discovered schedule. We evaluate this trade-off in Section~\ref{sec:experiments_pruning_strategy}.

%% file: sections/5_experiments.tex
\section{Implementation Setup}
\label{sec:experiments_setup}
\label{sec:experiments_implementation}

IOS is a framework-agnostic algorithm and can be implemented in popular frameworks. 
We implement the dynamic programming scheduling algorithm in Python and the execution engine in C++. 
The latency of a stage is directly measured in the execution engine to guide the scheduling. 
The execution engine is based on vendor-provided library cuDNN~\cite{cudnn} and supports operators' parallel execution.  
To concurrently execute multiple groups of operators, IOS puts different groups into different CUDA streams.
Kernels in different CUDA streams will be executed in parallel if there are enough computation resources. 
Throughout the experiments, we use cuDNN 7.6.5, cuda 10.2, NVIDIA driver 450.51.05, and adopt TensorRT 7.0.0.11 and TVM 0.7 as baseline libraries. 
\begin{table}[ht]
    \centering
    \small
    \begin{tabular}{c c c c}
        \hline
        Networks & \#Blocks & \#Operators & Operator Type \\
        \hline
        Inception V3 & $11$ & $119$ & Conv-Relu \\
        Randwire & $3$ & $120$ & Relu-SepConv \\
        NasNet & $13$ & $374$ & Relu-SepConv \\
        SqueezeNet & $10$ & $50$ & Conv-Relu \\
        \hline
    \end{tabular}
    \caption{The CNN benchmarks. Number of blocks, number of operators and the main operator type for each network are listed in the table. Here ``Conv-Relu" means a convolution followed by a ReLU activation and ``Relu-SepConv" means ReLU activation followed by seperatble convolution.}
    \label{tab:benchmark_networks}
    \vspace{-5pt}
\end{table}

We benchmark four modern CNNs in the experiment: Inception V3~\cite{szegedy2016rethinking}, RandWire~\cite{xie2019exploring}, NasNet-A~\cite{zoph2018learning} and SqueezeNet~\cite{iandola2016squeezenet}. Table \ref{tab:benchmark_networks} shows the number of blocks, the number of operators, and the main operator type for each network. IOS supports the user-defined schedule unit. In this experiment, we take the operator type shown in the table, besides other operators such as Concat, as the basic schedule unit.
Some models (e.g., ResNet~\cite{he2016deep}) might have limited inter-operator parallelization opportunities. For example, for ResNet-50 and ResNet-34, we can only achieve 2\% to 5\% speedup by paralleling the downsample convolutions. We do not consider it as our benchmarked model in the rest of the evaluation.   

\begin{figure*}[!ht]
    \centering
    \begin{minipage}[t]{0.49\textwidth}
        \centering
        \includegraphics[width=\textwidth]{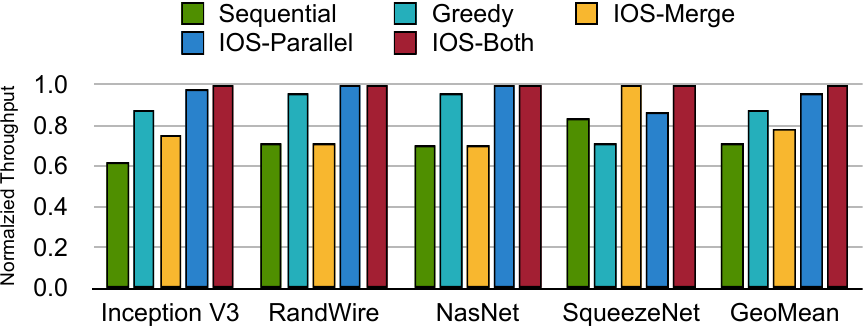}\vspace{-14pt}
        \caption{End-to-end performance comparison of different schedules across different CNNs on batch size one. The throughput is normalized to the best one for each model.}
        \label{fig:result_schedule}
    \end{minipage}%
    \hspace{7pt}
    \begin{minipage}[t]{0.49\textwidth}
        \centering
        \includegraphics[width=\textwidth]{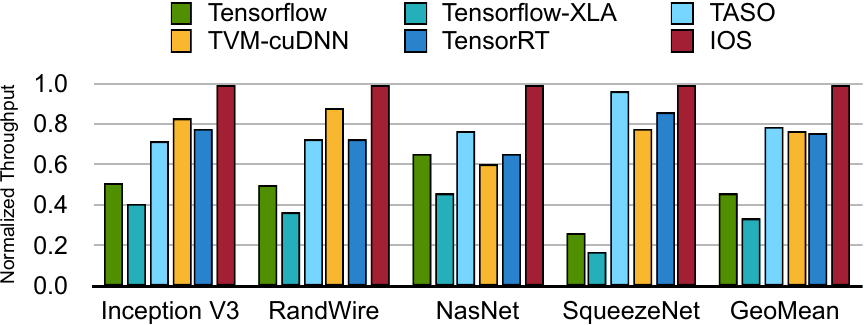}\vspace{-14pt}
        \caption{End-to-end performance comparison of
        different frameworks 
        across different CNNs on batch size one. The throughput is normalized to the best one for each model.}
    \label{fig:result_frameworks_cudnn}
    \end{minipage}
    \vspace{-5pt}
\end{figure*}   

We conduct each experiment $5$ times and report the average performance.
We adopt the schedule pruning strategy with $r=3$ and $s=8$ and conduct each experiment on NVIDIA Tesla V100 unless otherwise stated. Please refer to Appendix~\ref{sec:appendix_2080ti} for the experiments on other device. The IOS optimization cost for Inception V3 and SqueezeNet is less than $1$ minute and the IOS optimization cost for Randwire and NasNet is within $90$ minutes.

\section{Experiments}

\subsection{Comparison of Different Schedules}
\label{sec:experiments_schedules}
We first compare the inference performance among different schedules with batch size one. We compare five schedules: sequential schedule, greedy schedule, IOS-Merge schedule, IOS-Parallel schedule, and IOS-Both schedule.
The sequential schedule executes the operator one-by-one according to certain topological ordering.
The greedy schedule puts all the operators that can be executed currently in one stage, and repeats this process until all operators have been scheduled. IOS-Merge, IOS-Parallel, and IOS-Both schedules use the proposed approach to find the schedule but take different parallelization strategies. IOS-Merge only takes the ``operator merge" strategy. IOS-Parallel only takes the ``concurrent execution" strategy. 
IOS-Both considers both parallelization strategies.
All schedules are executed on IOS execute engine for a fair comparison.

Figure \ref{fig:result_schedule} shows that IOS-Both outperforms all the other four schedules. The greedy schedule gets good results on RandWire and NasNet. However, it degrades the performance of SqueezeNet because of the overhead of synchronization. Because we can not merge ``Relu-SepConv" operators in RandWire and NasNet, IOS-Merge gets the same schedule as Sequential, and IOS-Both gets the same schedule as IOS-Parallel. IOS-Both considers two parallelization strategies and outperforms all the other four schedules. In later experiments, ``IOS" refers to ``IOS-Both" by default.

\subsection{Comparison of cuDNN-based Frameworks}
\label{sec:experiments_frameworks}
For popular frameworks, there are two ways to exploit the intra-operator parallelism. Frameworks such as Tensorflow~\cite{tensorflow2015-whitepaper},  TASO~\cite{taso}, and TensorRT~\cite{tensorrt} use the vendor-provided library cuDNN. Frameworks such as TVM~\cite{chen2018tvm} and Ansor~\cite{ansor} search the tensor program schedule for each kernel. TVM also supports to call external libraries such as cuDNN to implement some kernels (e.g., convolution).  
In this subsection, we compare the performance of cuDNN-based frameworks with batch size one. Larger batch size is studied in the ablation study section.

There are five baselines: Tensorflow, Tensorflow-XLA, TASO, TVM-cuDNN, and TensorRT. Tensorflow-XLA is the tensorflow framework with XLA optimization turning on. 
TVM-cuDNN is the TVM framework that compiles a convolution neural network with cuDNN library, which would use the convolution kernel provided by cuDNN to execute convolutions. All other operators such as addition and concatenation would use their own kernels. 
For fair comparison, we only compare cuDNN-based libraries here. The comparison between TVM-AutoTune and IOS can be found in the ablation study section.
Figure \ref{fig:result_frameworks_cudnn} shows that IOS consistently outperforms all five baseline frameworks on four benchmark CNNs. IOS can achieve $1.1$ to $1.5\times$ speedup comparing to the state of the art library TASO, TVM-cuDNN, and TensorRT.

\subsection{More Active Warps Improve Utilization}
\label{sec:experiments_utilization}
\begin{figure}[!ht]
    \centering
    \includegraphics[width=0.90\linewidth]{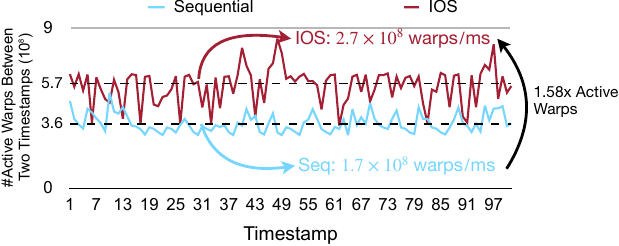}
    \vspace{-10pt}
    \caption{
    Active Warps for sequential schedule and IOS schedule. We use the model in Figure~\ref{fig:intro_overview}  in this experiment.
    }
    \vspace{-5pt}
    \label{fig:experiments_utilization}
\end{figure}
Model operators are mapped to GPU \emph{kernels} to execute. A kernel invokes a collection of \emph{threads} that are grouped into multiple \emph{thread blocks}.\footnote{We adopt the terminology used by NVIDIA.} 
Thread blocks are distributed to \emph{stream multiprocessors} (SMs). Each thread block on a SM is further partitioned into multiple \emph{warps}. A warp, as a basic execution unit, contains a fixed number of threads (e.g., 32 for NVIDIA GPU) to execute in a Single Instruction Multiple Thread (SIMT) fashion. 

A warp is considered \emph{active} from the time it is scheduled on an SM until it completes the last instruction. 
SM can hide the warps stall caused by memory accesses through fast context switching: at every cycle, each \emph{warp scheduler} will pick an eligible warp and issue instructions. If no eligible warp is available for a warp scheduler, the computation resources are underutilized. 
Increasing the number of active warps is an effective approach to increase the likelihood of having eligible warps to execute at each cycle.
Thus, it is \emph{crucial} to increase the number of active warps. 
Figure~\ref{fig:experiments_utilization} shows the number of active warps on the whole GPU throughout the repeated execution of both the IOS and the Sequential schedule, sampled using NVIDIA's CUPTI profiling toolset every 2.1 ms. IOS schedule achieves $58\%$ more active warps on average compared to the Sequential schedule. This explains the reason for IOS overall performance speedup.

\section{Ablation Study}

\subsection{Schedule Pruning Reduces Search Time}
\label{sec:experiments_pruning_strategy}
\begin{figure}[!ht]
    \centering
    \includegraphics[width=0.9\linewidth]{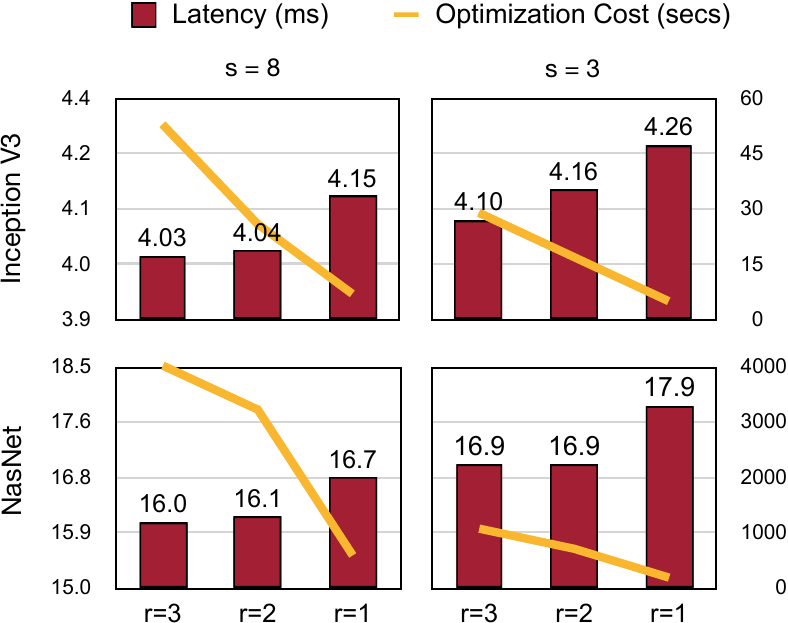}
    \vspace{-3pt}
    \caption{Trade-off between the optimized latency and the optimization cost for Inception V3 and NasNet. Two pruning strategy parameters $r$ and $s$ are used to prune the schedule space.
    $r$ limits the maximum number of operators in each group while $s$ limits the maximum number of groups in a stage. The left axis shows the optimized latency, and the right axis shows the optimization cost.}
    \label{fig:experiments_pruning}
\end{figure}

To explore the trade-off between optimized latency and optimization cost (i.e. search time), we experiment Inception V3 and NasNet with pruning strategy parameters $r = \{1, 2, 3\}$ and $s = \{ 3, 8\}$. 
As shown in Figure \ref{fig:experiments_pruning}, when $s$ and $r$ get smaller, the optimization cost decreases at the cost of larger network latency. 
This is because smaller $s$ and $r$ restrict the schedules that IOS explores, thus reduce the optimization cost and increase schedule latency. By setting $r=1$ and $s=8$, IOS still achieves $1.59\times$ and $1.37\times$ speedup for Inception V3 and NasNet, comparing to sequential schedule. Meanwhile, the optimization cost for each network is within $30$ seconds and $18$ minutes, respectively. 

\subsection{Specialized Scheduling is Beneficial}
\label{sec:experiments_specilization}
\begin{table}[ht]
    \centering
    
    \includegraphics[width=1.0\linewidth]{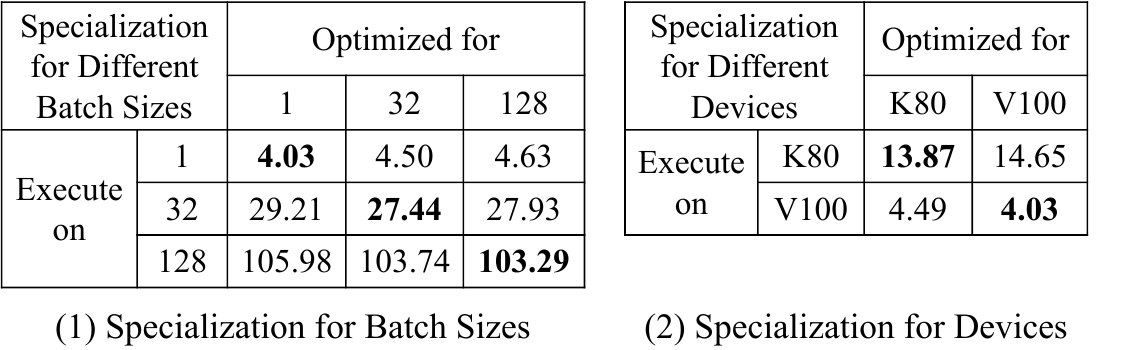}
    \vspace{-15pt}
    \caption{Latency (ms) of specialized schedules for batch size 1, 32 and 128, and specialized schedules for NVIDIA Tesla K80 and V100. The best performance is achieved when the schedule is specialized for each batch size and device. Each row is the batch size or device that the model is executed on. Each column is the batch size or device that IOS optimized for. InceptionV3 is used as a benchmark.}
    \label{tab:experiments_specialization_batchsize}
\end{table}
Different workloads (e.g. network with different batch sizes) have different computation features; thus it is necessary to specialize the schedule for different workloads. We optimize Inception V3 with batch size 1, 32 and 128. Then we execute the network with these schedules on batch size 1, 32 and 128 separately. In Table \ref{tab:experiments_specialization_batchsize} (1), the numbers in a row represents the latency executed with the same batch size but using schedules optimized for different batch sizes. 
The specialized schedule for each batch size achieved the best result.  
To explore the specialization for devices, we also optimize the network on both NVIDIA Tesla K80 and V100 with batch size one. Table \ref{tab:experiments_specialization_batchsize} (2) shows that the specialized schedule for each device also achieved better results.

\begin{figure}[h]
    \centering
    \includegraphics[width=1.0\linewidth]{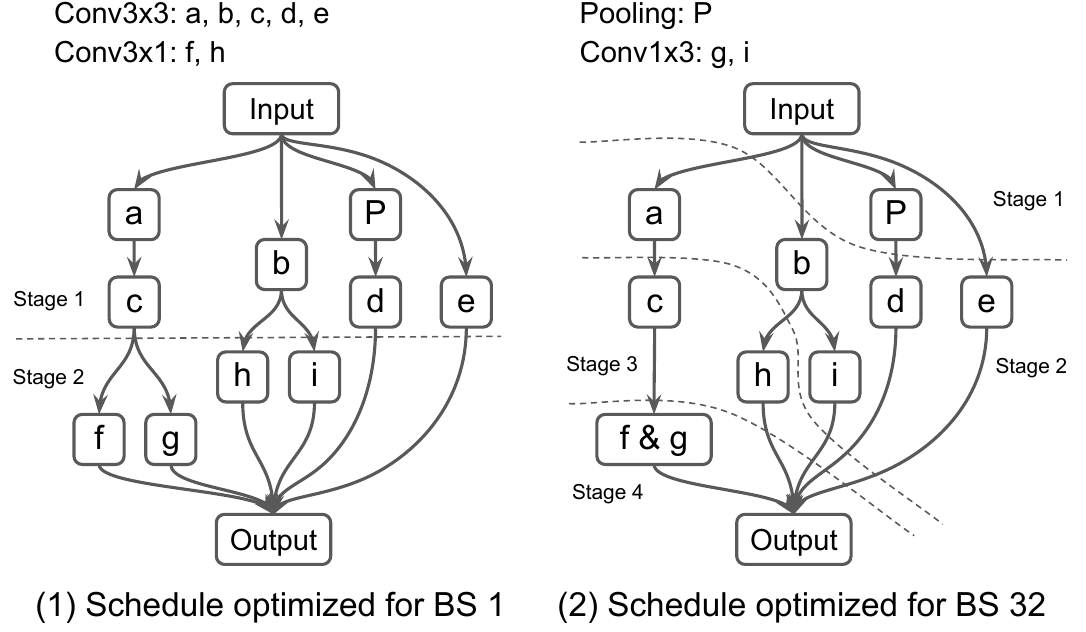}
    \vspace{-15pt}
    \caption{The schedule found by IOS for the last block of Inception V3. Operator a-e are convolution operator while operator P is the pooling operator. Schedule (1) and (2) are optimized for batch size 1 and 32, respectively. There are two stages in schedule (1) while there are 4 stages in schedule (2). Schedule (1) is 28\% faster than schedule (2) on batch size 1. Schedule (2) is 8\% faster than schedule (1) on batch size 32.}
    \label{fig:experiments_specialization_batchsize_demo}
\end{figure}

IOS discovers different schedules for different batch sizes. For example, Figure \ref{fig:experiments_specialization_batchsize_demo} shows the schedule of the last block of Inception V3 optimized for batch size 1 and 32, respectively. 
There are two stages in the schedule (1), which is optimized for batch size 1 while there are four stages in the schedule (2), which is optimized for batch size 32. 
The schedule (1) is 28\% faster than the schedule (2) on batch size 1, while the schedule (2) is 8\% faster than (1) on batch size 32. 
There are two differences between them. The first one is that convolution f and g in the schedule (2) are merged into a single convolution. 
This is because activation (the output tensor of an operator) is the memory bottleneck at large batch size. It is more crucial to reduce memory access, even at the cost of larger computation cost. 
Merging can reduce the memory access, because the merged kernel only access the output of convolution c once, instead of twice in the schedule (1).
However, because the kernel size of f and g are 3x1 and 1x3, respectively, their kernel size would be expanded to 3x3 by padding zeros, which increases the amount of computation.
Another difference between the schedule (1) and (2) is that the schedule (2) has more stages than the schedule (1). We found a similar phenomenon for large batch sizes because of resource contention. When multiple operators are executed on the device, there is a conflict over access to the shared resources such as the last-level cache, making the concurrent execution degrades the performance. This gets more severe for larger batch sizes because the demand for shared resources gets larger.

\subsection{Consistent Improvement for Different Batch Sizes}
\label{sec:experiments_larger_batchsize}
\begin{figure}[!h]
    \centering
    \includegraphics[width=0.95\linewidth]{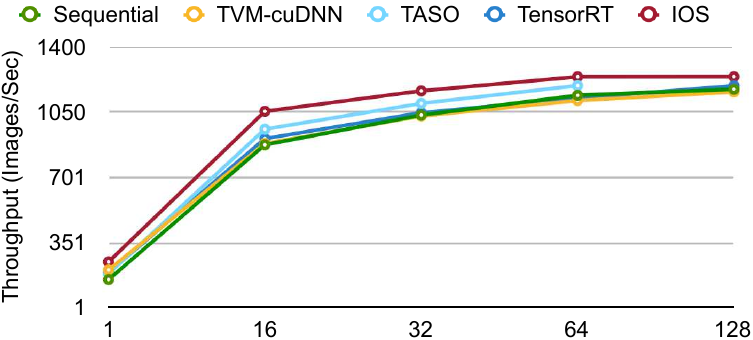}
    \vspace{-5pt}
    \caption{The throughput comparison of Sequential schedule, TVM-cuDNN, TASO, TensorRT and IOS on batch size 1 to 128 for Inception V3. TASO runs out of memory with batch size 128.}
    \label{fig:experiments_batchsize}
\end{figure}
In real-world applications, we need to handle different batch sizes for inference. For example, for real-time applications on edge devices, we usually use a batch size of one to reduce latency. In contrast, in cloud settings, the larger batch size is preferred to increase throughput. Changing the workload requires different inter-operator parallelization schedules. 
We optimize Inception V3 with the batch sizes of 1, 16, 32, 64, 128, and compare the throughput. Figure~\ref{fig:experiments_batchsize} shows that the throughput increases with the batch size. When the batch size is larger than 128, the performance saturates, and the throughput does not increase significantly anymore. 
The throughput of IOS outperforms all the baselines consistently on all batch sizes. 
Even though a larger batch size provides more data parallelism, we can still utilize inter-operator parallelism to further improve the throughput.

\subsection{Intra- and Inter-Operator Parallelism}
\label{sec:experiments_intra_inter}
\begin{figure}[H]
    \centering
    \includegraphics[width=1.0\linewidth]{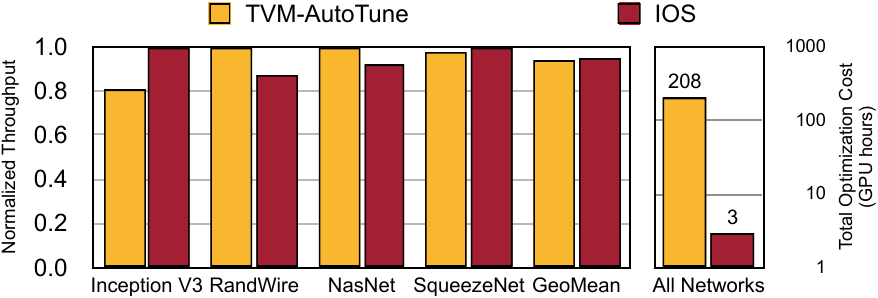}
    \vspace{-20pt}
    \caption{End-to-end performance comparison between TVM-AutoTune and IOS. TVM-AutoTune and IOS are {\em orthogonal} because TVM focuses on the intra-operator parallelism while IOS focuses on inter-operator parallelism. They can be combined to boost the inference performance further. The optimization cost of IOS is two orders of magnitude less than TVM.}
    \label{fig:expriments_tvmtune_ios}
    \vspace{-8pt}
\end{figure}

TVM exploits the intra-operator parallelism by searching the schedule for each kernel on a specific device.
IOS focuses on inter-operator parallelism and leaves the exploitation of intra-operator parallelism to cuDNN library.
Although intra- and inter-operator parallelism is {\em orthogonal} and can be combined, we compare TVM and IOS here to give some insight into each parallelism's benefit. 
As shown in Figure~\ref{fig:expriments_tvmtune_ios}, 
TVM takes $208$ GPU hours while IOS only takes $3$ GPU hours
to optimize the four networks.
IOS outperforms TVM on Inception V3 and SqueezeNet. 
This is because only utilizing intra-parallelism can not provide enough parallelism for the powerful computing device. 
Meanwhile, TVM outperforms IOS on Randwire and NasNet, 
because TVM finds more efficient kernels for separable convolutions, which occupy the majority of operators in Randwire and NasNet. 
We believe the combination of TVM and IOS would boost the performance further and leave this for future work.

%% file: sections/6_conclusions.tex
\section{Conclusion}

With the increasing computational capacity, the sequential execution of CNNs no longer provides sufficient parallelization opportunities to fully utilize all the computation resources. We propose IOS that combines intra- and inter-operator parallelism and adapt dynamic programming to find an efficient schedule that better utilizes the hardware. Experiments show that IOS can improve the GPU utilization and speedup modern CNN inference from $1.1$ to $1.5\times$ compared to 
the state-of-the-art libraries (e.g., TensorRT).

%% file: sections/7_acknowledgements.tex
\section*{Acknowledgements}

We want to thank Xiaodan (Serina) Tan for NVIDIA GPU-related issues and constructive discussion. 
This project was supported by the Canada Foundation for Innovation JELF grant, NSERC Discovery grant, AWS Machine Learning Research Award, Facebook Faculty Research Award, 
MIT-IBM Watson AI Lab, MIT Data Science and AI Lab (DSAIL), NVIDIA, and NSF CAREER Award \#1943349.

%% file: sections/8_appendix.tex
\appendix
\clearpage
\section{Proof of Time Complexity}
\label{sec:appendix_time_complexity}

In this section of appendix, we prove the time complexity bound given in Section \ref{sec:method_algorithm_complexity_and_pruning_strategy}. 
In Section \ref{sec:appendix_time_complexity_preliminary}, we give some preliminary definitions and theorems used in our proof. In Section \ref{sec:appendix_time_complexity_ios}, we prove the time complexity of inter-operator scheduler (IOS).

\subsection{Preliminary Definitions and Theorems}
\label{sec:appendix_time_complexity_preliminary}
In this subsection, we give the definition of chain and anti-chain, Dilworth's theorem~\cite{dilworth}, and a corollary, which is used in our proof later.

\begin{definition}[Chain and antichain] 
A \textit{chain} is a subset of a partially ordered set such that any two distinct elements in the subset are comparable. An \textit{antichain} is a subset such that any two distinct elements in the subset are incomparable.
\end{definition}

\begin{definition}[Chain decomposition of partial order set]
A \textit{chain decomposition} of a partial order set is a partition of the elements of the ordered set into disjoint chains.
\end{definition}

\begin{theorem}[Dilworth's Theorem]
In any finite partially ordered set, the largest antichain has the same size as the smallest chain decomposition. 
\end{theorem}

We apply the Dilworth's theorem to a directed acyclic graph and can get the following corollary.

\begin{corollary}
\label{corollary}
Let $G=(V, E)$ be a directed acyclic graph and $d$ be the width of $G$. We can decompose $V$ into $d$ sets such that any two vertices in the same set can be connected by a path in $G$.
\end{corollary}
\begin{proof}
Let $P=(V, E')$ be the partial order derived from $G$ by transitive closure. Then that two elements $u, v$ in $V$ are comparable in $P$ is equivalent to that there is a path between them in $G$. Thus, the width $d$ of $G$ equals the size of largest antichain of $P$. We apply the Dilworth's Theorem to $P$ and can get a decomposition of $V$ into $d$ chains in $P$: $S_1, S_2, \dots, S_d$. Because $S_i$ is a chain in $P$, any two elements in $S_i$ are comparable, which means there is a path bridge them in $G$. 
\end{proof}

\subsection{Time Complexity of Inter-Operator Scheduler}
\label{sec:appendix_time_complexity_ios}
In this subsection, we will prove the time complexity of IOS stated in Section ~\ref{sec:method_algorithm_complexity_and_pruning_strategy}. Then we will show that the upper bound can be reached by some computation graph.

\begin{lemma}
\label{lemma:ends}
If $S'_1$ ends $S$ and $S'_2$ ends $S-S'_1$, then $S'_1\cup S'_2$ also ends $S$ ($S'$ ends $S$ means that $S'$ is an ending of $S$).
\end{lemma}
\begin{proof}
We prove it by contradiction. If $S'_1\cup S'_2$ does not end $S$, there must exist $(u,v) \in E$ such that $u \in S'_1\cup S'_2$ and $v \in S-S'_1\cup S'_2$. Then we have $u\in S'_1$ or $u\in S'_2$. If $u \in S'_1$, we can get the contradiction that $S'_1$ is not an ending of $S$ because $v \in S-S'_1\cup S'_2 \subseteq S-S'_1$. If $u \in S'_2$, we can also get the contradiction that $S'_2$ is not an ending of $S-S'_1$ because $v \in S-S'_1\cup S'_2 = (S-S'_1)-S'_2$.
\end{proof}

\begin{lemma}
\label{lemma:S}
Let $S$ be a possible argument of \textproc{Scheduler}, we have $V-S$ ends $V$.
\end{lemma}
\begin{proof}
We can rewrite $S$ as $ S = V - \bigcup_{i=1}^{m}S'_i$, where $m \geq 0$ and $S'_k$ ends $V-\bigcup_{i=1}^{k-1}S'_i$ according to L17 in Algorithm~\ref{alg}. By repeating apply Lemma \ref{lemma:ends}, we can get that $\bigcup_{i=1}^m S'_i$ ends $V$, which means $V-S$ ends $V$.
\end{proof}

\begin{lemma}
    \label{lemma:seq}
    Let $V'$ be a subset of $V$ and any two operators in $V'$ are bridged by a path. Let $c$ be the size of $V'$. Then 
    \small
    $$
    |\{ (S\cap V', S'\cap V') \mid \text{$S'$ ends $S$, $V-S$ ends $V$ } \}| = \binom{c+2}{2}
    $$
    \normalsize
\end{lemma}
\begin{proof}
    Because any two operators in $V'$ is bridged by a path in $G$, operators in $V'$ are ordered sequentially. Because $V-S$ ends $V$, there are only $c+1$ possible sets of $S\cap V'$ because $S$ must be a prefix in the sequential ordered operators, including empty set. $S'\cap V'$ is a suffix of $S\cap V'$, including empty set. Then there are $\sum_{i=0}^{c}\sum_{j=0}^{i}1=\frac{(c+2)(c+1)}{2} = \binom{c+2}{2}$ possible pairs of $(S\cap V', S'\cap V')$.
\end{proof}

\begin{theorem}
The time complexity of inter-operator scheduler is $\CO(\binom{n/d+2}{2}^d)$, which can be relaxed to $\CO((\frac{n}{d}+1)^{2d})$, where $n$ is the number of operators in the computation graph and $d$ is its width.
\end{theorem}

\begin{proof}
We only need to count the number of pairs of $(S, S')$ that can reach L17 of Algorithm~\ref{alg} because L17-21 dominates the execution time of the scheduler, where $S$ is a subset of $V$ that is taken as the argument of \textproc{Scheduler} and $S'$ is an ending of $S$. By Lemma \ref{lemma:S}, $V-S$ ends $V$. By Corollary \ref{corollary}, we can decompose $V$ into $d$ disjoint partitions $V_1, V_2, \dots, V_d$ and any two operators $u, v$ in the same partition can be bridged by a path in $G$. We can build a one-to-one mapping that maps pair $(S, S')$ to $2d$-dimension tuple $(S\cap V_1, S'\cap V_1, \dots, S\cap V_d, S'\cap V_d)$ based on the partition. Then we only need to count the number of valid tuples to get the number of valid pairs. By Lemma \ref{lemma:seq}, the possible number of pairs $(S\cap V_i, S'\cap V_i)$ is $\binom{c_i+2}{2}$. Then an upper bound of the tuples is $\prod_{i=1}^d \binom{c_i+2}{2}$. It is an upper bound but not the exact number because currently we only consider the dependency inside each partition $V_i$ and ignored the dependency between different partitions. So the upper bound of the number of pairs of $(S, S')$ is $\prod_{i=1}^d \binom{c_i+2}{2}$. It can be relaxed to $\binom{n/d+2}{2}^d$ because $\sum_{i}^d c_i = n$ and it is maximized when $c_i$ are equal. For simplicity, it can be further relaxed to $(\frac{n}{d}+1)^{2d}$.
\end{proof}

\begin{figure}[ht]
    \centering
    \includegraphics[width=0.5\linewidth]{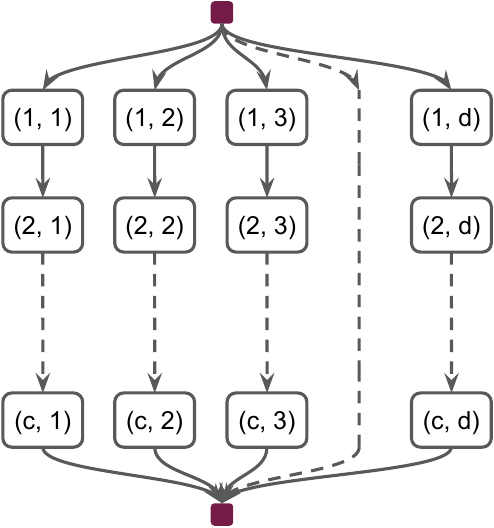}
    \caption{The example to make the time complexity $\CO(\binom{n/d+2}{2}^d)$ tight. The time complexity for this graph is $\CO(\binom{c+2}{2}^d)$}
    \label{fig:appendix_alg_example}
\end{figure}
The computation graph shown in Figure~\ref{fig:appendix_alg_example} is an example to demonstrate that the time complexity of $\CO(\binom{n/d+2}{2}^d)$ can be reached.

In this example, there are $d$ independent paths and each path has $c$ operators. Because the paths are independent with each other and there is no edge between two different paths, we can get the upper bound $\CO(\binom{c+2}{2}^d)$ by the analysis in above time complexity proof.

\section{Speedup on NVIDIA RTX 2080Ti}
\label{sec:appendix_2080ti}
\begin{figure}[h]
    \centering
    \includegraphics[width=\linewidth]{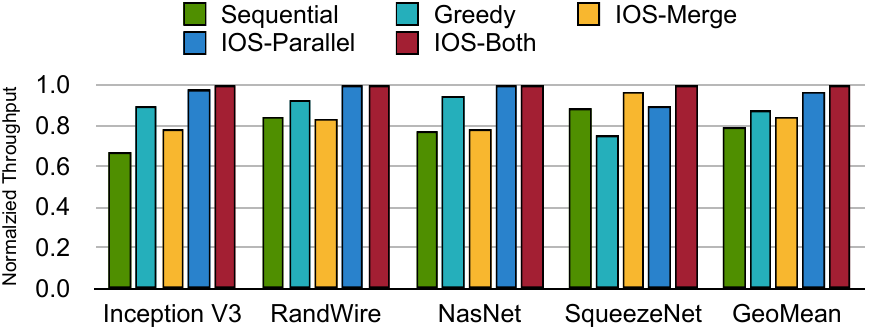}
    \vspace{-14pt}
    \caption{End-to-end performance comparison of different schedules across different CNNs on batch size one. The throughput is normalized to the best one for each model. This experiment is conducted on NVIDIA RTX 2080Ti.}
    \label{fig:appendix_schedule}
\end{figure}
\begin{figure}[h]
    \centering
    \includegraphics[width=\linewidth]{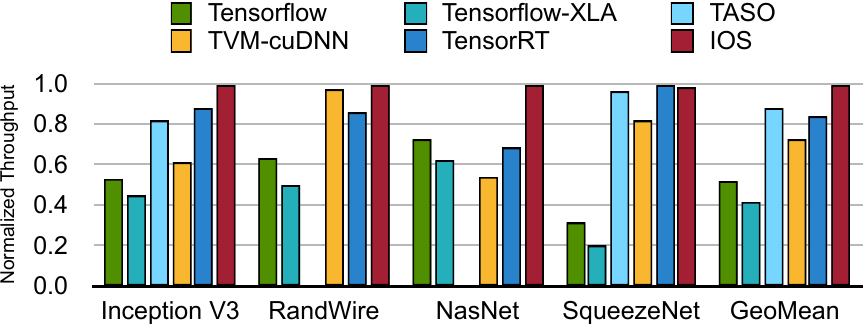} 
    \vspace{-14pt}
    \caption{End-to-end performance comparison of
    different frameworks 
    across different CNNs on batch size one. The throughput is normalized to the best one for each model. This experiment is conducted on NVIDIA RTX 2080Ti.}
    \label{fig:appendix_frameworks_cudnn}
\end{figure}

In addition to results on NVIDIA Tesla V100 (Volta architecture), we also conduct experiments on NVIDIA RTX 2080Ti (Turing architecture) to show that our optimization is generally effective across different GPU architectures. We use the same models and baselines for comparisons as in Section~\ref{sec:experiments_schedules} and Section~\ref{sec:experiments_frameworks}.

Figure~\ref{fig:appendix_schedule} shows that IOS with two parallelization strategies (i.e., IOS-Both) outperforms all other schedules. In particular, IOS-Both achieves $1.1\times$ to $1.5\times$ speedup comparing to the sequential schedule. 
Figrue~\ref{fig:appendix_frameworks_cudnn} shows that IOS outperforms all other cuDNN-based frameworks\footnote{TASO runs out of GPU memory on NVIDIA 2080Ti for RandWire and NasNet.} on Inception V3, RandWire, and NasNet. IOS achieves comparable performance with TASO and TensorRT on SquuezeNet. These results align with the results on V100.

\section{Block-wise Speedup}
\label{sec:experiments_blockwise}
\begin{figure}[h!]
    \centering
    \includegraphics[width=0.96\linewidth]{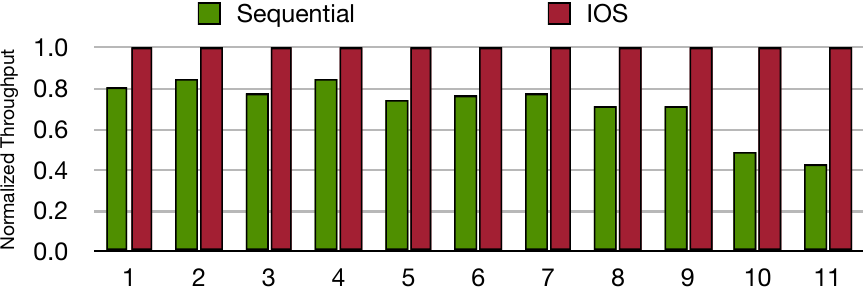}
    \caption{IOS consistently outperforms sequential executions on each block of Inception-v3. 
    }
    \label{fig:experiments_block_speedup}
\end{figure}

To explore the speedup for different blocks,  we compare the performance of each block of Inception-V3~\cite{szegedy2016rethinking} between sequential and IOS schedule (Figure~\ref{fig:experiments_block_speedup}). IOS consistently runs faster than the sequential schedule. The speedup for the individual block is up to 2.3$\times$, and the end-to-end speedup is 1.6$\times$. More speedup is achieved for back blocks because the width gets larger and more inter-parallelism is possible.